\documentclass{article}
\usepackage[utf8]{inputenc}

\usepackage[english]{babel}
\usepackage[T1]{fontenc}
\usepackage[margin=1in]{geometry} 
\usepackage[colorlinks=true, allcolors=blue]{hyperref}

\usepackage{nicefrac}
\usepackage{tikz}
\usepackage{amsmath}
\usepackage{amsthm}
\usepackage{amssymb}
\usepackage{mathtools}
\usepackage{graphicx}
\usepackage{subcaption}
\usepackage{amssymb}
\usepackage{bm}
\usepackage{xspace}
\usepackage{xcolor}
\usepackage{comment}
\usepackage{float} 
\usepackage{thmtools}
\usepackage{thm-restate}
\usepackage{cleveref}

\usepackage{xargs}                  
\usepackage[colorinlistoftodos,prependcaption,textsize=tiny]{todonotes}
\newcommandx{\unsure}[2][1=]{\todo[linecolor=red,backgroundcolor=red!25,bordercolor=red,#1]{#2}}
\newcommandx{\change}[2][1=]{\todo[linecolor=blue,backgroundcolor=blue!25,bordercolor=blue,#1]{#2}}
\newcommandx{\info}[2][1=]{\todo[linecolor=OliveGreen,backgroundcolor=OliveGreen!25,bordercolor=OliveGreen,#1]{#2}}
\newcommandx{\improvement}[2][1=]{\todo[linecolor=Plum,backgroundcolor=Plum!25,bordercolor=Plum,#1]{#2}}
\newcommandx{\thiswillnotshow}[2][1=]{\todo[disable,#1]{#2}}

\newtheorem{theorem}{Theorem}[section]
\newtheorem{lemma}[theorem]{Lemma}
\newtheorem{definition}[theorem]{Definition}

\newtheorem{remark}[theorem]{Remark}

\newtheorem{example}[theorem]{Example}

\newtheorem*{theorem*}{Theorem}
\newtheorem*{problem*}{Problem}

\DeclareMathOperator{\A}{\mathcal{A}\xspace}

\newcommand{\p}{p}
\newcommand{\aggr}{q}

\DeclareMathOperator{\states}{\mathcal{S}}

\DeclareMathOperator{\x}{\mathbf{x}\xspace}
\DeclareMathOperator{\y}{\mathbf{y}\xspace}

\DeclareMathOperator{\w}{\mathbf{w}\xspace}

\DeclareMathOperator{\opt}{\mathrm{OPT}\xspace}

\DeclareMathOperator{\lp}{\mathrm{LP}\xspace}

\DeclareMathOperator{\mlsd}{\mathrm{MLSD}\xspace}

\newcommand\event[1]{\mathop{\mathcal{X}\left(#1\right)}}

\newcommand\Ex[2]{\mathop{\underset{#1}{\mathbb{E}}\left[#2\right]}}

\newcommand\Prob[2]{\mathop{\underset{#1}{\mathbb{P}}\left(#2\right)}}

\usepackage[algo2e,ruled,noend]{algorithm2e} 

\SetCommentSty{mycommfont}

\crefname{algocfline}{Algorithm}{Algorithms}
\Crefname{algocfline}{Algorithm}{Algorithms}

\title{Last Switch Dependent Bandits with Monotone Payoff Functions}

\author{Ayoub Foussoul\\ 
Columbia University\\ 
\texttt{af3209@columbia.edu} 
\and 
Vineet Goyal\\ 
Columbia University\\ 
\texttt{vg2277@columbia.edu} \and Orestis Papadigenopoulos\\ 
Columbia University\\ 
\texttt{opapadig@columbia.edu}
\and 
Assaf Zeevi\\ 
Columbia University\\ 
\texttt{assaf@gsb.columbia.edu}
}

\date{}

\begin{document}

\maketitle

\begin{abstract}
In a recent work, Laforgue et al. introduce the model of {\em last switch dependent} (LSD) bandits, in an attempt to capture nonstationary phenomena induced by the interaction between the player and the environment. Examples include {\em satiation}, where consecutive plays of the same action lead to decreased performance, or {\em deprivation}, where the payoff of an action increases after an interval of inactivity. In this work, we take a step towards understanding the approximability of {\em planning} LSD bandits, namely, the (NP-hard) problem of computing an optimal arm-pulling strategy under complete knowledge of the model. In particular, we design the first efficient constant approximation algorithm for the problem and show that, under a natural monotonicity assumption on the payoffs, its approximation guarantee (almost) matches the state-of-the-art for the special and well-studied class of {\em recharging} bandits (also known as {\em delay-dependent}). In this attempt, we develop new tools and insights for this class of problems, including a novel {\em higher-dimensional relaxation} and the technique of {\em mirroring} the evolution of {\em virtual states}. We believe that these novel elements could potentially be used for approaching richer classes of action-induced nonstationary bandits (e.g., special instances of restless bandits). In the case where the model parameters are initially unknown, we develop an online learning adaptation of our algorithm for which we provide sublinear {\em regret} guarantees against its full-information counterpart.
\end{abstract}

\section{Introduction}
Shortly after the introduction of the (stochastic) {\em multi-armed bandits} (MAB) framework \cite{BB12,LS18}, practitioners and researchers quickly raised the issue of {\em nonstationarity}, thus questioning the restrictive assumption of the original model \cite{T33,LR85} that the environment (meaning the payoff distributions of the actions) remains intact. This opened the way to the development of important and well-studied extensions of the model, including {\em adversarial bandits} \cite{PCBFS02}, {\em reinforcement learning} \cite{jaksch2010near,szepesvariRL}, and more \cite{Keskin,Besbes,KNMS10,auer2019adaptively}. In many situations, a potential shift in the environment is not solely the result of external factors, rather than a natural consequence of its interaction with previously made decisions (see \cite{Whittle88} for examples). In an attempt to address this issue of {\em action-induced} nonstationarity various models have been proposed with {\em restless} \cite{Whittle88, TL12, GMS10} and {\em rested} \cite{Gittins79, TL12} bandits being the most prominent. In these settings, every arm is associated with a state-machine and its mean payoff depends on the current state. The state of each arm can change (potentially stochastically) at every round (in the restless case) or only after the arm is pulled (in rested case). Even ignoring the learning aspect and assuming complete knowledge of the underlying arm-state distributions, any attempt to compute a (near-)optimal {\em planning} policy -- usually via solving Bellman’s equations \cite{Bertsekas}) -- requires an exponentially large space in the number of actions and hits the wall of strong inapproximability results \cite{PT99}.

More recently, researchers have shifted their attention to special cases of restless bandits, which are simple enough to accept efficient (near-)optimal planning algorithms, yet expressive enough to capture fundamental action-induced nonstationary phenomena \cite{Levine, rising}. Immorlica and Kleinberg \cite{KI18} first attempt to model the effect of {\em deprivation} in online decision-making by introducing the model of {\em recharging} (a.k.a., {\em delay-dependent}) bandits. Here, the (mean) payoff of each action depends -- in an increasing fashion -- on the time elapsed since the action was played for the last time (often called {\em ``delay''}). Soon afterwards, a number of works focused on generalizations \cite{SLZZ21,papadigenopoulos2022nonstationary}, variations \cite{CCB19,PBG19}, and special cases \cite{BSSS19,PC21} of the model. Due to the computational hardness of the underlying planning problem, these works have a dual purpose: to construct an efficient near-optimal planning algorithm and, subsequently, to adapt it into an online learning policy for the case where the payoff distributions are unknown. 

\paragraph{Last switch dependent bandits.} In an attempt to capture a richer class of action-induced nonstationary phenomena, Laforgue, Clerici, Cesa-Bianchi, and Gilad-Bachrach \cite{LCCBGB21} recently introduced the model of {\em last switch dependent} (LSD) bandits. In their setting, the notion of ``delay'' -- of central role in recharging bandits -- is replaced by that of a ``{\em switch}'': a change in the course of action from the part of the decision-maker (see below). This shift of perspective not only strictly subsumes the recharging bandits model, but also captures additional natural behaviors, including that of {\em satiation}: the (gradual) degradation in performance due to the repeated use of resources. 

In this work, we study a variation of the LSD model, which we (informally) describe below:

\begin{problem*}[Last Switch Dependent Bandits with Monotone Payoffs ($k$-MLSD)] 
We consider a setting where the payoff of each arm is a function of its {\em state} at any given round. The state of an arm can be any (positive or negative) integer and changes, at the end of each round, according to the following rules: when an arm is at a positive state $\tau > 0$ and is not played at the current round, its state ``increases'' to $\tau+1$, while if it is played, its state transitions to $-1$. Dually, if an arm is at a negative state $\tau<0$, and is played at the current round, its state ``decreases'' to $\tau-1$, while if it is not played it transitions to $+1$. The payoff of each arm is a monotone non-decreasing function over the space of integer states. At each round, the decision-maker selects at most $k$ of the available arms and collects the sum of the associated payoffs (evaluated at the corresponding states). The objective is to maximize the total collected payoff within a (potentially unknown) time horizon.
\end{problem*}

The fundamental difference between the above model compared to its original formulation \cite{LCCBGB21} is the {\em monotonicity}. Specifically, we assume that the payoff function of each arm is monotone non-decreasing over the whole set of integer states, while in \cite{LCCBGB21} this assumption is only made for its negative part. Although the assumption excludes from the model any possible seasonal behaviors, our setting still widens the class of action-induced phenomena that can be captured (e.g., satiation), and still generalizes many existing works -- either strictly \cite{SLZZ21,papadigenopoulos2022nonstationary, BSSS19, PC21,KI18} or conceptually \cite{mintz2020nonstationary}. In addition, monotonicity is a plausible assumption in many situations where the positive effect of an action after a period of deprivation is higher than while in satiation (an everyday example is food consumption). Finally, we believe that monotonicity changes dramatically the approximability status of the problem; in fact, we conjecture that without this assumption the problem does not accept any polynomial-time constant approximations under standard complexity assumptions (yet proving it falls beyond the scope of this work). The validity of such a statement would justify the fact the algorithm developed in \cite{LCCBGB21} is not efficient and the provided guarantees involve additive losses.

\paragraph{Summary of contributions.} In this work, we provide the first polynomial-time $\mathcal{O}(1)$-approximation algorithm for the problem of planning LSD bandits with monotone payoff functions. Interestingly, the approximation guarantee of our algorithm matches (up to an arbitrarily small error) the state-of-the-art for the special case of recharging bandits \cite{papadigenopoulos2022nonstationary}. An immediate practical implication of our work is that one can replace the recharging model with the (strictly more expressive) $k$-$\mlsd$ one without any sacrifice in the approximability (compared to the state-of-the-art until this work). Moreover, compared to \cite{LCCBGB21}, our algorithm can also handle the case where more than one arms can be played at each round with gradually improved provable performance, and does not dependent on the time horizon. Finally, we complement our results by developing an online bandit adaptation of our algorithm and proving that the latter achieves (efficiently) sublinear regret in the regime where the payoff functions are initially unknown.

We address the reader to \Cref{apx:bibl} for a more technical discussion on the related work. 

\subsection{Technical Challenges and Roadmap}
The main contribution of this work is an efficient LP-based approximation algorithm for the problem of planning $k$-$\mlsd$ bandits. In particular, our algorithm collects (asymptotically and in expectation) a $(1-\epsilon)\left(1- \nicefrac{k^k}{e^k k!}\right)$-fraction\footnote{Using Stirling's formula one can show that, for large $k$, the long-run approximation guarantee of our algorithm behaves roughly as $(1-\epsilon) \left(1 - \frac{1}{\sqrt{2 \pi k}} \right)$.} of the optimal expected payoff in time $\text{poly}(n, \tau^{\max}, \nicefrac{1}{\epsilon})$, where $n$ is the number of arms and $\tau^{\max}$ the maximum saturation time of any payoff function (formally defined in \Cref{sec:preliminaries}). The fact that the above guarantee (almost) matches the state-of-the-art for the case of recharging payoffs stems from our attempt to generalize the best-known framework for the latter \cite{papadigenopoulos2022nonstationary}. Extending the existing framework to the case of $k$-$\mlsd$ requires 
several novel technical elements and insights, which we outline below. 

\paragraph{Continuous relaxation based on recurrent intervals.} A safe takeaway from the existing literature on recharging bandits \cite{KI18, SLZZ21, papadigenopoulos2022nonstationary} is the use of continuous relaxations -- a technique which seemingly facilitates the design (and/or analysis) of near-optimal approximation algorithms for this kind of dynamic planning problems. Another common element in these works is that, due to computational and practical reasons (e.g., large or unknown time horizon, respectively), such a relaxation should have a {\em fluidized} form; that is, one which (approximately) recovers the ``optimal'' frequencies for playing each individual arm under a given delay, rather than time-accurate playing sequences. After obtaining such information as a starting point, the role of an algorithm becomes to construct a feasible playing schedule which closely approximates these frequencies. However, as opposed to the special case of recharging bandits, the (two-sided) state transitions in $k$-$\mlsd$ present an asymmetry: after a nonplay-play switch an arm transitions to a different state compared to that of a play-nonplay one. This fact seems to preclude the design of tight continuous relaxations based on the frequencies of playing each arm under a given state, as in \cite{KI18, SLZZ21, papadigenopoulos2022nonstationary}, thus, posing a significant technical hurdle. 

In \Cref{sec:relaxation}, we develop a continuous LP-based relaxation of slightly increased dimension, based on the novel notion of {\em recurrent intervals}.

\paragraph{Online rounding via mirroring virtual evolutions.} 
Starting from an optimal solution to the relaxation, the high-level goal becomes to ``round'' it into a feasible arm-pulling schedule, where each arm is played at a pattern ``close'' to the one indicated by the corresponding variables of the LP. A critical issue that emerges from any such rounding is that of {\em collisions}: situations where trying to mimic the ``optimal'' pattern of all arms requires playing more than $k$ arms in some rounds. For the case of recharging bandits, a way to minimize the effect of collisions among the arms is the technique of {\em interleaved scheduling}, presented in \cite{papadigenopoulos2022nonstationary}. In that simpler setting, after sampling a unique delay for each arm from the corresponding LP relaxation, the arm is allowed to be played only in rounds which are integer multiples of this delay. The effect of collisions is controlled by adding a uniformly random offset to each of the above subsequences, aiming to avoid adversarial worst-case scenarios. Combined with other elements (e.g., the {\em correlation gap} of uniform matroid rank function), this ensures that the expected average payoff collected consists a constant fraction of the relaxation (and, hence, the optimal). In the $k$-$\mlsd$ setting, however, the recurrent intervals involve richer arm-playing patterns compared to a unique periodic play of \cite{papadigenopoulos2022nonstationary} -- a fact which complicates the implementation of the above technique. 

In \Cref{sec:planning}, we design a planning algorithm for the $k$-$\mlsd$ problem and provide an analysis of its approximation guarantee. To achieve this, we generalize the randomized rounding procedure in \cite{papadigenopoulos2022nonstationary} and introduce the technique of {\em mirroring virtual evolutions}; the latter allows us to overcome the issue of collisions via abstracting the interleaved scheduling technique and extending it to the $k$-$\mlsd$ setting.

\paragraph{Sample complexity and online learning.} In \Cref{sec:learning}, we present an adaptation of our algorithm for the case where the payoff functions are unknown and the decision-maker receives noisy semi-bandit feedback on the selected arms. After observing that our planning algorithm constructs feasible schedules starting from any (possibly suboptimal) solution to our LP, we start by providing sample complexity results for approximating the latter. Using these we construct a bandit adaptation of our algorithm, based on a combination of Explore-then-Commit and the doubling trick, for which we provide sublinear regret guarantees. 

All the omitted proofs of our results have been moved to the Appendix.

\section{Problem Definition and Notation}
\label{sec:preliminaries}

We consider a set $\A = [n]$ of {\em arms} (or {\em actions}) and an unknown time horizon of $T$ rounds. Each arm $i \in [n]$ is associated with a (mean) {\em payoff function} $\p_i(\cdot): \states \rightarrow [0,1]$ defined over a set of {\em states}, where $\p_i(\tau)$ denotes the mean payoff of arm $i$ when played under state $\tau$. The set of states is the same for all arms and coincides with that of all integers (positive and negative) excluding $0$, namely, $\states = \mathbb{Z} \setminus \{0\}$. At the beginning of each round $t$, the state of each arm $i$, denoted by $\tau_i(t) \in \states$, is a function of the state and the action taken in the previous round (see below for more details on state transitioning). At each round, a decision-maker can play any subset of at most $k < n$ arms and collect the sum of associated payoffs, each drawn from a distribution of mean given by its payoff function, evaluated at its current state. The planning objective is to maximize the cumulative expected payoff collected in $T$ rounds.

\paragraph{Payoff functions.} For every arm $i \in \A$, we assume that its payoff function is: (a) monotone {\em non-decreasing}, i.e., for every two states $\tau_1, \tau_2 \in \states$ with $\tau_1 < \tau_2$, it holds that $p_i(\tau_1) \leq p_i(\tau_2)$. Further, (b) we assume that the payoff function satisfies the {\em finite saturation} property, namely, there exist known integers $\tau^{\min}_i < 0 < \tau^{\max}_i$ such that $p_i(\tau) = p_i(\tau^{\max}_i)$ for every $\tau > \tau^{\max}_i$ and $p_i(\tau) = p_i(\tau^{\min}_i)$ for every $\tau < \tau^{\min}_i$. For simplicity of exposition, we assume without loss of generality that all arms have the same (upper and lower) saturation times, given by $\tau^{\max} = \max_{i \in \A} \tau^{\max}_i$ and $\tau^{\min} = \min_{i \in \A} \tau^{\min}_i$. We remark that payoff functions are given {\em explicitly} as part of the problem input and, hence, the running time of efficient algorithms can be polynomial in $\tau^{\max}$ and $|\tau^{\min}|$.

\paragraph{State transitions.}
For every arm $i \in \A$, the state at the beginning of round $t$ is given by the following rule: $\tau_i(1) = 1$ and $\tau_i(t+1) = \delta_i(\tau_i(t), A_t)$, where $A_t \subseteq \A$ is the subset of arms played at round $t$ and $\delta_i(\cdot, \cdot)$ is the {\em state transition} function, defined as
\begin{align}
    \delta_i(\tau, S) = \left\{
	\begin{array}{ll}
		\tau - 1 & \mbox{if } i \in S \text{ and } \tau < 0 \\
		 1 & \mbox{if } i \notin S \text{ and } \tau < 0 \\
		 -1 & \mbox{if } i \in S \text{ and } \tau > 0 \\ \label{eq:statestransition}
		\tau + 1 & \mbox{if } i \notin S \text{ and } \tau > 0
	\end{array}.
\right.
\end{align}
Intuitively, the state of each arm $i$ denotes the time passed since the arm last took place in a switch of actions. In particular, a positive state $\tau > 0$ denotes that arm $i$ has not been played for the last $\tau$ rounds. Thus, as time progresses and $i$ is not played, its state increases by $1$ at each round until the first time where $i$ is played again, at which point its state transitions to $-1$. Similarly, a negative state $\tau < 0$ denotes that arm $i$ has been played for $-\tau$ consecutive time steps. If an arm is played at some state $\tau < 0$ its state decreases by $1$ while, if it is not played, it transitions to $+1$.

We remark that the choice of initial state of $1$ for every arm is made for mathematical convenience and does not affect our results qualitatively.

\paragraph{Technical notation.} For any integer $q$, we use the notation $[q] = \{1,2,\ldots,q\}$. For any vector $\x \in \mathbb{R}^n$ and set $S \subseteq [n]$, we denote $\x(S) = \sum_{i \in S} x_i$. For any $\x \in [0,1]^n$, let $\mathcal{D}(\x)$ denote some distribution over $2^{[n]}$ whose marginal probabilities are given by $x_i = \Prob{S \sim \mathcal{D}(\x)}{i \in S}$ for all $i \in [n]$. Similarly, let $\mathcal{I}(\x)$ denote the element-wise independent distribution over $2^{[n]}$ with marginals $\x=(x_i)_i$, that is, the distribution such that sampling $S \sim \mathcal{I}(\x)$ consists of adding each element $i \in [n]$ to $S$ independently with probability $x_i$. For any $l, u \in \mathbb{Z}$ such that $u \geq l$ we denote by $\states^{u}_{l} = \states \cap [l, u]$ the subset of states ranging from $l$ to $u$, while we define $\states^{u}_{l} = \emptyset$ in the case where $l > u$.

\section{LP Relaxation Based on Recurrent Intervals} \label{sec:relaxation}

As we have already discussed, the idea of developing a relaxation based on the fraction of time an arm is played under a specific state (as in \cite{KI18,SLZZ21,papadigenopoulos2022nonstationary}) does not quite work for the $k$-$\mlsd$ setting (see \Cref{apx:badrelaxations} for a discussion). Instead, we manage to construct a continuous relaxation of the planning problem, based on the novel notion of a {\em recurrent interval}: a minimal sequence of states (and admissible actions) which start and end (by transitioning) to the same state. 
We begin this section by formally defining a {recurrent interval} and showing that -- for the particular case $k$-$\mlsd$ -- this structure exhibits two useful properties: (a) any optimal solution can be (almost) described as a concatenation of valid recurrent intervals for each arm and (b) every such interval has a succinct representation. Subsequently, by leveraging the above properties, we construct an LP-based (approximate and asymptotic) relaxation of slightly increased dimension which serves as the starting point of our planning algorithm.

\subsection{Recurrent Intervals and Aggregated Payoffs} 

The construction of a tight relaxation for $k$-$\mlsd$ requires a change of perspective: instead of measuring the fraction of time an arm is played under a specific state, we rather consider the fraction of time the arm spends for completing a specific (cyclic) pattern in the trajectory of its states. In order to formalize the above idea, we introduce the notion of {\em recurrent intervals}, which plays a central role in the design of our relaxation and algorithm:

\begin{definition}[Recurrent Intervals]
For any arm and given states $u \in \mathbb{Z}_+$ and $l \in \mathbb{Z}_-$, a {\em recurrent interval}, denoted by $I(u,l)$, is a sequence of distinct states (and associated actions) during which the arm starts from state $+1$ and moves back to the same state after a number of rounds. In particular, starting from state $+1$, the arm is not played until it reaches state $u>0$, where it is played for the first time (thus, transitioning to state $-1$). Then, the arm is consecutively played for $|l+1|$ rounds (including state $l+1$). The recurrent interval is terminated by not playing the arm at state $l$ (thus, transitioning back to $+1$). 
\end{definition}

The above definition allows us to study the problem from a perspective of cyclic sub-sequences instead of individual actions. In the following definition, we summarize several characteristics of any recurrent interval that are useful for the description and design of our algorithm:

\begin{definition}\label{def:trajectory} The {\em characteristic trajectory} of a recurrent interval $I = I(u,l)$ is a function $\beta_{I} : \states^u_{l} \rightarrow \{\bullet, \bot	\}$, such that $\beta_I(\tau) = \bullet$ for every $\tau \in \states^{-1}_{l+1} \cup \{u\}$ and $\beta_I(\tau) = \bot$ for every $\tau \in \states^{u-1}_{1} \cup \{l\},$ where $\bullet$ and $\bot$ represent the play and non-play of the arm, respectively. The {\em transition function} of a recurrent interval $I = I(u,l)$ is a function $\delta_I:\states^u_{l} \rightarrow \states^u_{l}$, such that
$$
\delta_{I}(\tau) = \left\{
	\begin{array}{ll}
		\tau - 1 & \mbox{if } \beta_{I}(\tau) = \bullet ~\text{ and } \tau < 0 \\
		 1 & \mbox{if } \beta_{I}(\tau) = \bot \text{ and } \tau < 0 \\
		 -1 & \mbox{if } \beta_{I}(\tau) = \bullet ~\text{ and } \tau > 0 \\ 
		\tau + 1 & \mbox{if } \beta_{I}(\tau) = \bot \text{ and } \tau > 0
	\end{array}.
\right.
$$
Finally, the {\em length} of a recurrent interval $I(u,l)$, namely, the number of consecutive rounds it occupies (including those where the arm is not played), is given by 
$\|I\|_{ri} = u - l.$ Notice that an arm is actually pulled $-l$ times during $I(u,l)$.
\end{definition}

The characteristic trajectory of a recurrent interval $I=I(u,l)$ satisfies the following property: if one starts an arm from any state in the interval, namely $\tau \in \states \cap [l,u]$, then by repeatedly playing the arm if and only if $\beta_I(\tau) = \bullet$, the trajectory of its states will follow the periodic pattern indicated by $I$. The transition function of $I$, on the other hand, gives the state to which an arm must transition from a given state $\tau$, when it follows the actions indicated by its characteristic trajectory. Note that every recurrent interval can be mapped to a distinct characteristic trajectory (resp., transition function) and the opposite. Hence, \Cref{def:trajectory} can also serve as an alternative and equivalent definition of recurrent intervals. 

Another important characteristic of a recurrent interval relative to a specific arm is the total expected payoff of the involved states.

\begin{definition}[Aggregated Payoff] The {\em aggregated} (expected) {\em payoff} of a recurrent interval $I(u,l)$ for an arm $i \in \A$ is given by the function $\aggr_i(u, l) = p_i(u) + \sum^{-1}_{\tau = l+1} p_i(\tau).$
\end{definition}

From the perspective of each single arm any feasible arm-pulling schedule can be decomposed into a sequence of (potentially varying) recurrent intervals. The only exceptions to the above rule are limiting scenarios where an arm is either never or constantly played until time horizon ends (thus, the last recurrent interval is interrupted). 

The following lemma evaluates the loss from restricting the set of recurrent intervals of each arm to only $I(u,l)$ with $l \geq \tau^L$ for some $\tau^L \leq -1$:

\begin{restatable}{lemma}{restatenearopt} \label{lem:nearopt}
For any instance of $k$-$\mlsd$ and any $\tau^L \leq -1$, there exists a (deterministic) near-optimal solution, where the sequence of plays and non-plays of every arm consists of a concatenation of recurrent intervals of the form $I(u,l)$ with $l \geq \tau^L$, potentially followed by a sequence of non-plays until the end of the time horizon. The total payoff collected by this solution, denoted by $\overline{\opt}(T)$, satisfies
$$
\overline{\opt}(T) \geq \left(1 - \frac{1}{1-\tau^L}\right) \cdot \opt(T) - n,
$$
where $\opt(T)$ is the optimal payoff in $T$ rounds.
\end{restatable}

We remark that, by applying \Cref{lem:nearopt} for $\tau^L = -1$, one can reduce the problem to an instance of recharging bandits. Specifically, any $\gamma$-approximation algorithm for the latter (e.g., any of \cite{SLZZ21,papadigenopoulos2022nonstationary}) implies a long-run $\nicefrac{\gamma}{2}$-approximation algorithm for $k$-$\mlsd$. This approach, however, would cost an additional $\nicefrac{1}{2}$-factor in the approximation compared to the algorithm we develop in \Cref{sec:planning}.

\subsection{Definition and Properties of LP Relaxation}

For any instance of $k$-$\mlsd$ and integer $\tau^L \leq -1$, we consider the following LP relaxation:
\begin{align}
\max_{\x \succeq \bf 0} 
&~ \sum_{i \in \A}\sum_{u \in \states^{\tau^{\max}}_{1}}\sum_{l \in \states^{-1}_{\tau^L}} \aggr_i(u,l) \cdot x_{i,u,l}  \label{lp:LP} \tag{\textbf{LP}}\\
\text{s.t.}& \sum_{i \in \A}\sum_{u \in \states^{\tau^{\max}}_{1}}\sum_{l \in \states^{-1}_{\tau^L}} -l \cdot x_{i,u,l} \leq k, \label{lp:total}\tag{C.1}\\
&\sum_{u \in \states^{\tau^{\max}}_{1}}\sum_{l \in \states^{-1}_{\tau^L}} (u - l) \cdot x_{i,u,l} \leq 1, \forall i \in \A \label{lp:arm} \tag{C.2}.
\end{align}
In the above formulation each variable $x_{i,u,l}$ represents the fraction of time where arm $i$ participates in a recurrent interval $I(u,l)$ in a feasible solution. Constraints \eqref{lp:total} originate from the fact that, when at most $k$ arms can be played at each round, the (total) fraction of time any arm is pulled during any recurrent interval cannot be more than $k$ (recall, an arm is played exactly $-l$ times during $I(u,l)$). Constraints \eqref{lp:arm} hold due to the fact that, for every arm $i \in \A$, the various recurrent intervals it participates in cannot be overlapping in any feasible solution, by definition.

As we show in the following result, \eqref{lp:LP} (approximately and asymptotically) yields an upper bound on the optimal average expected payoff:

\begin{restatable}{lemma}{restaterelaxation} \label{lem:relaxation}
For any instance of $k$-$\mlsd$, let $\opt(T)$ be the optimal payoff collected in $T$ rounds. For the optimal value of \eqref{lp:LP}, denoted by $\lp^*$, it holds
$$
T \cdot \lp^* \geq \left(1 - \frac{1}{1-\tau^L}\right)\opt(T) - n.
$$
\end{restatable}

\section{Design and Analysis of the Planning Algorithm} 
\label{sec:planning}
Recall that each variable of our LP relaxation represents the fraction of time a specific pair of arm and recurrent interval occurs in an optimal solution and, hence, it has a higher dimension compared to the relaxations used in \cite{KI18,SLZZ21,papadigenopoulos2022nonstationary}. However, this increased dimension comes with a loss of any useful structure in the fractional solution returned, which makes its interpretation significantly harder compared to \cite{papadigenopoulos2022nonstationary} (there, the sparsity pattern of extreme point solutions almost reveals the ``optimal'' frequencies). 

In this section, we show how to overcome the above issue by proposing a novel planning algorithm and analyzing its approximation guarantees. We remark that in the planning setting we can assume w.l.o.g. that payoffs are deterministic, since their realizations do not affect the trajectory of the algorithm in any way.

\subsection{Description of the Algorithm and Main Result} \label{sec:algorithm}

At a high-level, the algorithm starts by sampling a {\em unique} recurrent interval for each arm using the information obtained by the relaxation. This is achieved by generalizing the rounding in \cite{papadigenopoulos2022nonstationary} and selecting each recurrent interval with marginal probability proportional to its corresponding LP variable and its length (number of distinct states). After sampling a unique recurrent interval for each arm, our algorithm constructs a fictitious copy of its state (called ``virtual'' state) which periodically evolves over the sampled interval. Notice, of course, that every such evolution implies a unique trajectory of actions (plays or non-plays) for each arm. Critically, the algorithm initiates the evolution of the virtual state over the corresponding recurrent interval from state chosen uniformly at random. At each round, our algorithm first selects the arms whose virtual state indicates that they must be played in order to remain within their virtual periodic trajectory. Among these arms, the algorithm plays the $k$ (or less) which maximize the total expected payoff collected, evaluated at the corresponding virtual states. 

Let $\x^*$ be an optimal solution to \eqref{lp:LP}, for some parameter $\tau^L$. Notice that, due to the fact that the payoff functions are provided explicitly as part of the problem instance, such a solution can be computed efficiently. Given that the produced solution $\x^*$ is generally fractional, our algorithm then proceeds in two main phases: the {\em initialization} (a.k.a. {\em offline}) phase, which in turn includes the steps of {\em randomized rounding} and {\em mirroring}, and the {\em online} phase. Each of these steps are described below in more detail. 

\paragraph{Randomized rounding.}
Given an optimal fractional solution $\x^* = \{x^*_{i,u,l}\}$ to \eqref{lp:LP}, the first step of our algorithm -- as part of its initialization -- is to select a (unique) recurrent interval for every arm. This is achieved via randomized rounding and is performed once and offline (i.e., before any arm-pulling). Specifically, for every arm $i \in \A$, the algorithm randomly samples a unique recurrent interval $I(u, l)$, with marginal probability $(u - l) \cdot x^*_{i,u,l}$ or no recurrent interval at all with probability $1-\sum_{u \in \states^{\tau^{\max}}_{1}}\sum_{l \in \states^{-1}_{\tau^L}} (u - l) \cdot x^*_{i,u,l}$. By constraints \eqref{lp:arm} of \eqref{lp:LP}, for each arm $i$ it holds that $\sum_{u \in \states^{\tau^{\max}}_{1}}\sum_{l \in \states^{-1}_{\tau^L}} (u - l) \cdot x^*_{i,u,l} \leq 1$ and, hence, the above sampling procedure is well-defined. In case no recurrent interval is sampled for some arm, then the arm is never played by the algorithm. At the end of this phase, each arm $i \in \A$ is associated with at most one recurrent interval, which we denote by ${I}_{i} = I(u_i, l_i)$. 

\paragraph{Mirroring virtual evolutions.} Let $\A' \subseteq \A$ be the subset of arms for which a recurrent interval is sampled during the rounding phase. For every arm $i \in \A'$, the algorithm defines an evolution of a ``{\em virtual}'' state, as a function of its sampled interval $I_i$. In particular, the virtual state of each arm evolves (periodically) over an infinite concatenation of copies of $I_i$. Notice that, by definition of a recurrent interval, the produced sequence of states is periodic and corresponds to a unique periodic sequence of actions (plays or non-plays) that can implement it. Afterwards, the algorithm randomly interleaves the sequences of virtual states of the arms by forcing each one to start a random number of steps into the future. Critically, this random number (called {\em offset}) is uniformly chosen from $r_i \sim \{0, \ldots, \|I_i\|_{ri}-1\}$, in a way that the virtual state can start (at time $t=1$) from any state involved in $I_i$, equiprobably. Let $\nu_i(t) \in \states_{l_i}^{u_i}$ denote the virtual state of arm $i \in \A'$ at time $t$.

\paragraph{Online phase.} \Cref{algo} then proceeds to its online phase. At any round $t$, the algorithm first constructs a set $C_t$ of {\em candidate} arms, which contains the set of arms $i \in \A'$ that satisfy $\beta_{I_i}\left(\nu_i(t)\right) = \bullet$, where $\nu_i(t)$ denotes its virtual state for the same round. In other words, the algorithm considers an arm $i \in \A'$ a candidate, if its virtual state would require a play in order to remain in the periodic trajectory induced by $I_i$. Then, the algorithm plays the (at most) $k$ arms of highest mean payoffs evaluated at their virtual states. 

As we describe in \Cref{algo}, all the above steps (including the simulation of the evolution of the virtual states) can be performed online and, hence, the algorithm does not require knowledge of the time horizon. Moreover, by setting $\tau_L = - \lceil \nicefrac{1}{\epsilon} \rceil$, it is easy to verify that the running-time of the algorithm is $\text{poly}(n, \tau^{\max},\frac{1}{\epsilon})$. The approximation guarantee of \Cref{algo} is summarized in the following Theorem, which is the main result of this work:

\begin{restatable}{theorem}{restatemain} \label{thm:main}
For any instance of $k$-$\mlsd$ and any fixed $\epsilon \in (0,1)$, the total expected payoff collected by \Cref{algo} in $T$ rounds is at least 
$$
(1-\epsilon) \left(1 - \frac{k^k}{e^k k!} \right) \opt(T) - \mathcal{O}\left(n + \tau^{\max} \cdot k\right),
$$
where $\opt(T)$ is the optimal expected payoff that can be collected. 
\end{restatable}

Notice that using Stirling's formula the multiplicative factor in the guarantee of \Cref{thm:main} can be closely approximated by $(1-\epsilon) \left(1 - \frac{1}{\sqrt{2 \pi k}} \right)$, for large enough $k$. Finally, before we proceed to the proof of \Cref{thm:main}, we remark that the approximation guarantees we provide are tight for any $k$ (modulo the $(1 - \epsilon)$-factor): 

\begin{remark}
For any $k$, there exists an instance of $k$-$\mlsd$ where the approximation guarantee of \Cref{algo} is asymptotically equal to $1 - \frac{k^k}{e^k \cdot k!}$. This implies that our analysis in \Cref{thm:main} is tight (up to the $(1-\epsilon)$-factor). We address the reader to \Cref{sec:tightexample} for additional details. 
\end{remark}

\begin{algorithm2e} \label{algo}
\DontPrintSemicolon
\caption{Planning $k$-$\mlsd$ bandits.}
    Compute an optimal solution $\x^*$ to \eqref{lp:LP} with parameters $\tau^{\max}$ and $\tau^L = - \lceil \nicefrac{1}{\epsilon} \rceil$ for $\epsilon \in (0,1)$.\;
    Set $\A' \gets \emptyset$.\;
    \For{{each arm} $i \in \A$}{
    Sample a unique recurrent interval $I_{i}=I(u_i, l_i)$ with marginal probability $(u_i - l_i)\cdot x^*_{i, u_i, l_i}$. \;
    \If{a recurrent interval is sampled}{ 
    Draw $r_i \sim \mathcal{U}\{0,1, \ldots,\|I_i\|_{ri}-1\}$.\;
    Initialize the {\em virtual} state $\nu_i(0) \gets 1$. \; 
    \For{$\ell = 1,2, \ldots, r_i$}{
        $\nu_i(0) \gets \delta_{I_i}(\nu_i(0))$.
    }
    Add $i$ to $\A'$.\;}
    }
    \For{$t = 1,2, \ldots$}{
    \For{{each arm} $i \in \A'$}{
        $\nu_i(t) \gets \delta_{I_i}(\nu_i(t-1))$.
    }
    Let $C_t \subseteq \A'$ be the subset of {\em candidate} arms, defined as $C_t = \left\{i \in \A' \mid \beta_{I_i}\left(\nu_i(t)\right) = \bullet \right\}.$ \;
    Play the subset of $k$ arms in $C_t$ of maximum total expected payoff evaluated at the virtual states: 
    $$A_t = \underset{S \subseteq C_t, |S| \leq k}{\text{argmax}} \sum_{i \in S} p_i(\nu_i(t)).\;$$
}
\end{algorithm2e}

\subsection{Analysis of the Approximation Guarantee}
\label{sec:analysis}

Before we present the analysis of \Cref{algo} -- which leads to the proof of \Cref{thm:main} -- we remark that the correctness of the algorithm follows immediately from the fact that at most $k$ arms are played at each round. In order to lower-bound the total expected payoff of \Cref{algo}, we focus on analyzing the expected payoff collected at any fixed round. Our goal is to show that the latter consists a constant fraction of the average optimal expected payoff, that is, $\nicefrac{\opt(T)}{T}$. Having established that, the proof simply follows by linearity of expectation.

\paragraph{Reduction to virtual states.} Let us fix any time step $t \in [T]$ with $t \geq \tau^{\max}$ and let us denote by $\A'$ be the subset of arms for which a recurrent interval has been sampled during the randomized rounding step of our algorithm. Recall that, at time $t$, the algorithm plays the $k$ arms in the candidate set $C_t$ (or fewer, if $|C_t| < k$) with the highest expected payoff, evaluated at their virtual states. The first step of our proof is to show that the actual expected payoff collected, i.e., the one evaluated at the actual states, is at least as much as the one evaluated at the virtual ones. This is implied immediately by the following stronger result:

\begin{restatable}{lemma}{restatecomparestate}
\label{lem:comparestate}
For every round $t \geq \tau^{\max}$ and arm $i \in \A'$, it deterministically holds that $\tau_i(t) \geq \nu_i(t)$. By monotonicity of the payoff functions, this further implies that $p_i(\tau_i(t)) \geq p_i(\nu_i(t))$.
\end{restatable}

By applying the above Lemma, we can lower-bound the expected payoff collected at round $t$ as follows: 
\begin{align}
\Ex{}{\max_{S \subseteq C_t, |S| \leq k} \sum_{i \in S} p_i(\tau_i(t))} \geq \Ex{}{\max_{S \subseteq C_t, |S| \leq k} \sum_{i \in S} p_i(\nu_i(t))}, \label{eq:virtualstates}
\end{align}
where the expectation above is taken over the randomness of sampling recurrent intervals and the offsets. 

\paragraph{Candidate triples and their distribution.} Inequality \eqref{eq:virtualstates} allows us to reduce the analysis to studying the evolution of only the virtual states of the arms. For any fixed time step $t \geq \tau^{\max}$, the set of candidate arms $C_t$ contains the subset of arms $i \in \A'$ that are played at their virtual state $\nu_i(t)$, namely, those which satisfy $\beta_{I_i}(\nu_i(t)) = \bullet$. This set, however, does not give any information regarding the actual recurrent interval sampled nor the virtual state under which an arm is a candidate. 

For this reason, we extend the notation by introducing the following more expressive set, which we refer to as the set of {\em candidate triples}:
$$
\mathcal{T}_t = \{(i, I(u,l), \nu) \;|\; i \in C_t,\; I(u,l) = I_i,\; \nu = \nu_i(t)\}.
$$
Further, for every arm $i \in \A$, let $$U^i = \left\{(i, I(u , l), \nu) \;|\; u \in \states^{\tau^{\max}}_1 ,\; l \in \states^{-1}_{\tau^{L}},\; \beta_{I(u,l)}(\nu) = \bullet \right\}$$ denote the set of all possible candidate triples involving arm $i$, and let $U = \bigcup_{i \in \A} U^i$ denote the set of all possible candidate triples for all arms. Finally, let, $\mathcal{T}^i_t = \mathcal{T}_t \cap U^i$
denote the set of candidate triples of time $t$ involving arm $i$.

At this point, we observe that the set $\mathcal{T}^i_t$, for any arm $i \in \A$, is distributed independently of other arms. This is because each $\mathcal{T}^i_t$ is a function of the sampled recurrent interval $I_i$ and the choice of the random offset $r_i$, and each of these quantities is drawn independently for every arm. In addition, we observe that the event that a specific triple $(i,I(u,l),\nu)$ belongs to $\mathcal{T}^i_t$ is mutually-exclusive to that of any other triple corresponding to the same arm. Indeed, at most one recurrent interval $I_i$ can be sampled for arm $i$ in the offline phase and -- assuming one was indeed sampled -- the arm can be in exactly one virtual state at each time.  

The above discussion motivates the definition of the following class of distributions: 

\begin{definition}[Block-Mutually-Exclusive]
Consider a ground set $[m]$ of elements and a given partition into subsets $V^1, \ldots, V^n$, such that $V^i \cap V^j = \emptyset$, for every $i \neq j$, and $\bigcup_{j \in [n]} V^j = [m]$. A distribution $\mathcal{D}$ over $2^{[m]}$ is called block-mutually-exclusive, if it first samples at most one element from each subset $V^i$ (independently of other subsets), using a marginal distribution $\mathcal{D}^i$ over the subsets of $2^{V^i}$ of at most one element (i.e., the singletons and the empty set), and then return the union of the sampled elements. We refer to the distributions $\mathcal{D}^1, \ldots, \mathcal{D}^n$ as {\em block-marginals}.
\end{definition}

In other words, in a block-mutually-exclusive distribution, the sampling of elements across different subsets $V^i$ and $V^j$ for $i \neq j$ is independent, but within each set it is mutually-exclusive.

The following lemma characterizes the distribution of the set of candidate triples $\mathcal{T}_t$:
\begin{restatable}{lemma}{restatedistribution} \label{lem:distribution}
    For any fixed time step $t$, the distribution of $\mathcal{T}_t$, denoted by $\mathcal{C}(\x^*)$, is block-mutually-exclusive with block-marginals $\mathcal{C}^1(\x^*), \ldots, \mathcal{C}^n(\x^*)$ where, for every $i \in \A$, $\mathcal{C}^i(\x^*)$ is a distribution over singletons of $2^{U^i}$ and the empty set. Further, for every $i \in \A$ and $(i, I(u,l), \nu) \in U^i$, it holds $\Prob{S \sim \mathcal{C}^i(\x^*)}{S = \{(i, I(u,l), \nu)\}} = x^*_{i, u, l}$ and $\Prob{S \sim \mathcal{C}^i(\x^*)}{S = \emptyset} = 1-\sum_{(i,I(u,l), \tau) \in U^i} x^*_{i, u, l}$.
\end{restatable}

Due to \Cref{lem:distribution}, the RHS in \eqref{eq:virtualstates} can be thus written as a function of $\mathcal{C}(\x^*)$ as follows:
\begin{align}
\Ex{}{\max_{S \subseteq C_t, |S| \leq k} \sum_{i \in S} p_i(\nu_i(t))}
=
\Ex{C \sim \mathcal{C}(\x^*)}{\max_{S \subseteq C, |S| \leq k} \sum_{(i, I(u,l), \nu) \in S} p_i(\nu)}. \label{eq:triples}
\end{align}

\paragraph{Reaching the LP upper bound via the correlation gap.} The next step in our analysis is to associate the RHS in the equality \eqref{eq:triples} with the optimal solution of \eqref{lp:LP}. In this direction and following the paradigm of \cite{papadigenopoulos2022nonstationary}, we first observe that the function $g_k: 2^U \rightarrow [0, \infty)$, defined as
$$ 
g_k(C) = \max_{S \subseteq C, |S| \leq k} \sum_{(i, I(u,l), \nu) \in S} p_i(\nu),
$$
is monotone non-decreasing submodular\footnote{A function $f:2^{[n]} \rightarrow [0,\infty)$ is submodular if for every $S, T \subseteq [n]$ it holds that $f(S) + f(T) \geq f(S \cup T) + f(S \cap T)$.}. In particular, $g_k(C)$ can be thought of as an instance of the weighted rank function of the (rank-$k$) uniform matroid, defined as follows: given a weight vector $\w \in [0, \infty)^m$ over a ground set $[m]$ of elements, the weighted rank function of the $k$-uniform matroid is given by $f_{\w,k}(S) = \max_{\substack{I \subseteq S \\ |I| \leq k}} \; \sum_{i \in I} w_i.$

For any set function $f:2^{[m]} \rightarrow [0, \infty)$ and vector $\y \in [0,1]^m$, one can define the following standard continuous extensions: (a) the {\em multilinear extension} $F: [0,1]^m \rightarrow [0, +\infty)$ given by $F(\y) = \Ex{S \sim \mathcal{I}(\y)}{f(S)},$ and (b) the {\em concave closure} $f^+: [0,1]^m \rightarrow [0, \infty)$ given by $f^+(\y) := \sup_{\mathcal{D}(\y)}\Ex{S \sim \mathcal{D}(\y)}{f(S)}.$ 

Let us denote $\gamma_k = \left(1-\frac{k^k}{e^k k!}\right)$. For the particular case of $f_{\w,k}$, the following result is known:

\begin{lemma}[Correlation Gap \cite{yan11}]
\label{lem:correlationgap}
Let $f_{\w,k}: 2^{[m]} \rightarrow [0,+\infty)$ be the weighted rank function of the rank-$k$ uniform matroid. Then, for any $\y \in [0,1]^m$, we have 
$$f^+_{\w,k}(\y) \geq F_{\w,k}(\y) \geq \gamma_k \cdot f^+_{\w,k}(\y).$$
\end{lemma}

Using the above, one could potentially associate the RHS of equality \eqref{eq:triples} with the concave closure of $g_k$, denoted by $g_k^+(\x^*)$, thus moving a step closer to the end goal of recovering the optimal value of the LP. However, the definition of multilinear extension assumes that expectations are taken by including each element {\em independently} to a random set with marginals given by $\x^*$, which is not the case in our setting. 

We resolve the above issue by showing that, for the particular case of block-mutually-exclusive distributions, as $\mathcal{C}(\x^*)$, the expectation of $g_k(C)$, can only decrease if the elements were instead added to $C$ independently, but with the same marginals.

\begin{restatable}{lemma}{restatesubmodgap} \label{lem:submodulargap}
Let $f:2^{[m]} \rightarrow [0, +\infty)$ be a submodular set function. For any $\y \in [0,1]^m$ and block-mutually-exclusive distribution $\mathcal{C}(\y)$ with marginals $\y$, we have
$$\Ex{S \sim \mathcal{C}(\y)}{f(S)} \geq \Ex{S \sim \mathcal{I}(\y)}{f(S)}.$$
\end{restatable}

We remark that the above result allows us to leverage any known correlation gap result for (not necessarily monotone) submodular functions for settings where the underlying distribution is block-mutually-exclusive. 

The final step is to relate the concave closure $g^+_k(\x^*)$ to the optimal value of \eqref{lp:LP}. 

\begin{restatable}{lemma}{restatesuptolp}
\label{lem:relationtolp}
    Let $\lp^*$ denote the optimal value of \eqref{lp:LP}. Then, we have $g^+_k(\x^*) \geq \lp^*.$
\end{restatable}

By combining the above elements, we can now complete the proof of \Cref{thm:main} (see \Cref{apx:approx}).

\section{Online Learning Adaptation with Sublinear Regret}
\label{sec:learning}

We now consider the online learning variant of $k$-$\mlsd$ where the mean payoff functions are unknown and need to be learned online. In particular, the payoff from playing arm $i$ at state $\tau$ is now drawn independently from a distribution with unknown mean $p_i(\tau)$ and bounded values in $[0,1]$. We assume that the player observes the payoff realization of each individual arm played (semi-bandit feedback). The goal is to design a bandit algorithm with sublinear regret measured against $(1-\epsilon)\cdot \gamma_k \cdot \opt(T)$, namely, the total expected payoff guaranteed (asymptotically) by \Cref{algo}. More formally, we seek to minimize the following approximate regret: $$\text{Reg}(T) = (1-\epsilon) \cdot \gamma_k \cdot  \opt(T) - R(T),$$ where $R(T)$ is the total expected payoff collected by the online learning algorithm over the whole time horizon. 

We develop a bandit adaptation of \Cref{algo} based on Explore-then-Commit, assuming that the time horizon is known. By a standard application of the doubling trick, this assumption only comes at a cost of an additional polylogarithmic factor in the regret. Our main theorem is summarized below and its proof can be found in \Cref{apx:learning}:

\begin{restatable}{theorem}{restateregret} \label{thm:regret}
There exists a bandit adaptation of \Cref{algo} for $k$-$\mlsd$ with regret upper-bounded as
$$
\mathcal{O}\left(n^{\frac{1}{3}}k^{\frac{2}{3}}\left((\tau^{\max})^2+\frac{1}{\epsilon}\right)^{\frac{1}{3}}\ln^{\frac{1}{3}}\left(\left(\tau^{\max}+\frac{1}{\epsilon}\right)T\right) \cdot T^{\frac{2}{3}} + k \tau^{\max} + n\right).
$$
\end{restatable}

\section*{Conclusion and Further Directions}
In this work, we consider the model of last switch dependent bandits, recently introduced in \cite{LCCBGB21}. We provide the first polynomial-time constant approximation algorithm for the planning problem under the additional assumption of monotonic payoffs, which already generalizes several studied models \cite{BSSS19,SLZZ21,papadigenopoulos2022nonstationary}. Our algorithm relies on novel techniques and insights that might be of independent interest including a novel relaxation, the concept of recurrent intervals, and the technique of mirroring virtual evolutions. This work leaves a number of interesting future directions: (a) we conjecture that the monotonicity of the payoff functions is necessary for the problem to accept constant approximations. Thus, an immediate direction would be to provide inapproximability results for the problem in the absence of this assumption. (b) The LP-based and time-correlated nature of our algorithm significantly complicates its adaptation into a {\em dynamic} online learning policy (e.g., based on UCB \cite{ACBF02}). Providing a learning adaptation of improved (order-optimal) regret is, hence, another interesting direction. Finally, (c) we believe that the ``mirroring'' technique we develop in this work could be used for tackling even richer subclasses of {\em restless} bandits which satisfy some analogous form of monotonicity.

\bibliographystyle{alpha}
\bibliography{ref.bib}
\newpage

\onecolumn

\appendix

\section{Background and Related Work} \label{apx:bibl}
In their original work on recharging bandits, Immorlica and Kleinberg \cite{KI18} construct a $(1-\epsilon)$-PTAS for the setting where -- in addition to monotone non-decreasing -- the payoffs are concave and Lipschitz functions of the delay. This is achieved through an elegant combination of partial enumeration and the novel technique of (randomized time-correlated) {\em interleaved rounding} of a concave relaxation. More related to our work, Simchi-Levi et al. \cite{SLZZ21} drop the assumptions of concavity and Lipschitzness and replace it with the conceptually weaker assumption of finite recovery (i.e., after a specific delay threshold the payoffs stabilize). They provide a $\nicefrac{1}{4}$-approximation algorithm by defining and searching over the space of {\em purely periodic policies}; the guarantees of their algorithm improve when more than one arms can be played per round. Additional variations of the problem have been studied \cite{CCB19,PBG19} with the special case of {\em blocking bandits} receiving particular attention \cite{BSSS19,BCMT20,APBCS21,PC21,BPCS20}. In blocking bandits each arm has a fixed mean payoff, yet becomes unavailable for a known number of rounds after each play. 

The state-of-the-art planning algorithm for recharging bandits (under the assumptions made in \cite{SLZZ21}) is due to Papadigenopoulos et al. \cite{papadigenopoulos2022nonstationary}. In their work, the authors construct a $\left(1 - \nicefrac{k^k}{e^k k!}\right)$-approximation for the case where at most $k$ arms can be pulled at each round. Their algorithm is based on rounding the solution of a natural linear-programming (LP) relaxation of the problem, through a combination of randomized rounding and the technique of {\em interleaved scheduling}. The idea of the algorithm we develop in this paper is based on abstracting and extending the above techniques.

Beyond introducing the model of LSD bandits, Laforgue et al. \cite{LCCBGB21} provide an additive approximation to the optimal solution with sublinear regret, for the case where at most one arm can be played per round. At a high-level, their approach is based on partitioning the time horizon into blocks of reasonable size and then finding an optimal arm-playing sequence within each block. However, due to the interaction between consecutive blocks (through the arms' states), one has to increase the size of each block in order to shrink the (additive) error, which makes the computation of this optimal sequence computationally hard.

\section{LP Relaxation Based on Recurrent Intervals: Omitted Proofs}

\label{apx:relaxation}

\restatenearopt*
\begin{proof}
Consider an instance $I$ of $k$-$\mlsd$. A deterministic solution $\pi(I)$ for $I$ consists of a schedule of plays and non-plays of each arm $i$ over the time horizon $T$, such that the total number of plays at each time step is at most $k$. We denote by $\pi_i(I)$ the sequence of plays and non-plays of arm $i$ under the solution $\pi(I)$. Note that $I$ always has an optimal solution that is deterministic, and let $\pi^*(I)$ be one such solution. 

Consider now a feasible solution $\Bar{\pi}(I)$ such that for every arm $i$, the sequence $\Bar{\pi}_i(I)$ is produced using $\pi^*_i(I)$ as follows: starting from $t=1$, each time we encounter a sequence of $1-\tau^L$ consecutive plays, we omit the $(1-\tau^L)$-th play (i.e., we turn it to a non-play). Further, we omit the last play of $\pi^*_i(I)$, if such a play exists. Clearly, $\Bar{\pi}(I)$ is a feasible solution, since the total number of played arms at each time step remains at most $k$. 

By construction of $\Bar{\pi}(I)$, for every arm $i$ the sequence $\Bar{\pi}_i(I)$ consists of a concatenation of recurrent intervals $I(u, l)$ such that $l \geq \tau^L$, potentially followed by a sequence of non-plays. In fact, any other recurrent interval $I(u, l)$ with $l < \tau^L$ must contain a sequence of $1-\tau^L$ consecutive plays for arm $i$ and, thus, cannot be contained in $\Bar{\pi}_i(I)$. Further, $\Bar{\pi}_i(I)$ ends with a non-play which implies that it can be split into recurrent intervals starting from $t=1$ and ending each one when a play/non-play switch is encountered. To conclude the proof, we show that $\Bar{\pi}(I)$ has an expected payoff
$$
\overline{\opt}(T) \geq (1 - \frac{1}{1-\tau^L}) \cdot \opt(T) - n.
$$
Fix an arm $i$ and let $\opt_i(T)$ denote the payoff that the optimal solution collects from arm $i$. Let $t_1, \ldots, t_R, t_{R+1}$ denote the time steps of the omitted plays from $\pi^*_i(I)$ with $t_{R+1}$ being the time step of the last play of $\pi^*_i(I)$ (assuming one exists). Let $\Delta_i = \sum_{r=1}^{R+1} p_i(\tau^*_i(t_r))$ denote the total payoff of the omitted steps, where $\tau_i^*(t)$ is the state of arm $i$ at time $t$ in schedule $\pi^*_i(I)$. Apart from the play $t_{R+1}$, every other omitted play is preceded by $-\tau^L$ plays in $\pi^*_i(I)$. By monotonicity of the payoff functions, for every $r \in \{1, \ldots, R\}$ we have 
$$p_i(\tau_i^*(t_r)) \leq \frac{1}{-\tau^L} \sum_{t=t_r+\tau^L}^{t_r-1} p_i(\tau_i^*(t)).
$$ 
By construction, the sets of time steps $\{t_1+\tau^L, \ldots, t_1-1\}, \ldots, \{t_R+\tau^L, \ldots, t_R-1\}$ which are the sets of $-\tau^L$ time steps preceding $t_1, \ldots, t_R$ respectively are disjoint. Hence
\begin{align*}
    \Delta_i = \sum_{r=1}^{R+1} p_i(\tau_i^*(t_r)) \leq  \frac{1}{-\tau^L} \sum_{r=1}^{R} \sum_{t=t_r+\tau^L}^{t_r-1} p_i(\tau_i^*(t)) + 1 \leq  \frac{1}{-\tau^L} (\opt_i(T) -\Delta_i) + 1,
\end{align*}
which implies that
\begin{align*}
    \Delta_i \leq \frac{1}{1-\tau^L}\opt_i(T) + \frac{-\tau_L}{1-\tau^{L}} \leq \frac{1}{1-\tau^L}\opt_i(T) + 1.
\end{align*}
Finally, since omitting a play can only make the state of the arm (and hence the payoff) higher in subsequent rounds, the gain from following the schedule $\Bar{\pi}_i(I)$ is at least $\opt(T)-\sum_{i=1}^n \Delta_i$. Hence
\begin{align*}
    \overline{\opt}(T) 
    &\geq 
    \opt(T) - \sum_{i=1}^n \Delta_i \\
    &\geq \opt(T) - \frac{1}{1-\tau^L} \sum_{i=1}^n \opt_i(T) - n\\
    &= (1 - \frac{1}{1-\tau^L}) \opt(T) - n.
\end{align*}
\end{proof}

\restaterelaxation*
\begin{proof}
Consider an instance $I$ of $k$-$\mlsd$ and let $\Bar{\pi}(I)$ be some deterministic solution such that, for every arm $i$, the schedule $\Bar{\pi}_i(I)$ consists of a concatenation of only recurrent intervals $I(u, l)$ with $l \geq \tau^L$, potentially followed by a sequence of non-plays. Let $\overline{\opt}(T)$ denote the total payoff of $\Bar{\pi}(I)$ over $T$ rounds. In order to prove the Lemma, it suffices to construct a feasible solution to \eqref{lp:LP} with objective value (at least) $\frac{\overline{\opt}(T)}{T}$; indeed, by \Cref{lem:nearopt}, this would imply that
$$ T \cdot \lp^* \geq \left(1 - \frac{1}{1-\tau^L}\right)\opt(T) - n.
$$
Let $\Bar{N}_i(u,l)$ be the number of times the recurrent interval $I(u,l)$ appears in the schedule $\Bar{\pi}_i(I)$ starting from $t=1$. Consider $\Bar{\x} \in \A\times \states_{1}^{\tau^{\max}} \times \states_{\tau^{L}}^{-1}$ such that for every $i\in \A$, $u \in \states_{1}^{\tau^{\max}}$ and $l \in \states_{\tau^{L}}^{-1}$ we have
    \begin{align*}
        \Bar{x}_{i,u,l} := \left\{
	\begin{array}{ll}
		\frac{\Bar{N}_i(u,l)}{T} & u < \tau^{\max}\\
		\sum_{u' \geq u} \frac{\Bar{N}_i(u',l)}{T}  & u = \tau^{\max}
	\end{array}.
        \right.
    \end{align*}
    Since an arm is pulled exactly $-l$ times during an interval $I(u,l)$, the total number of plays of arm $i$ under the solution $\Bar{\pi}(I)$ is 
    $\sum_{u > 0}\sum_{l \in \states_{\tau^L}^{-1}} -l \cdot \Bar{N}_i(u,l)$
    and, hence, the total number of plays of all arms over the whole time horizon satisfies 
    $$\sum_{i \in \A} \sum_{u > 0}\sum_{l \in \states_{\tau^L}^{-1}} -l \cdot \Bar{N}_i(u,l) \leq k T,$$
    by feasibility of $\Bar{\pi}(I)$. 
    Therefore, we have that
    \begin{align*}
        \sum_{i \in \A}\sum_{u \in \states^{\tau^{\max}}_{1}}\sum_{l \in \states^{-1}_{\tau^L}} -l \cdot \Bar{x}_{i,u,l} 
        &= \sum_{i \in \A}\sum_{u \in \states^{\tau^{\max}-1}_{1}}\sum_{l \in \states^{-1}_{\tau^L}} -l \cdot  \frac{\Bar{N}_i(u,l)}{T} + \sum_{i \in \A}\sum_{l \in \states^{-1}_{\tau^L}} \sum_{u \geq \tau^{\max}} -l \cdot  \frac{\Bar{N}_i(u,l)}{T}\\
        &=\frac{1}{T}\sum_{i \in \A} \sum_{u > 0}\sum_{l \in \states_{\tau^L}^{-1}} -l \cdot \Bar{N}_i(u,l) \leq k,
    \end{align*}
    which implies the feasibility of $\Bar{\x}$ w.r.t. constraints \eqref{lp:total}. 
    
    Notice that the length of the recurrent interval $I(u,l)$ is $\|I\|_{ri} = u - l$ and, thus, for every arm $i$ the total number of steps occupied by all the recurrent intervals in $\Bar{\pi}_i(I)$ is given by $\sum_{u>0}\sum_{l \in \states_{\tau^L}^{-1}} (u-l) \Bar{N}_i(u,l),$ which is less or equal to $T$ by construction. Therefore, for every arm $i$, we have
    \begin{align*}
        \sum_{u \in \states^{\tau^{\max}}_{1}}\sum_{l \in \states^{-1}_{\tau^L}} (u-l) \cdot \Bar{x}_{i,u,l} 
        &= \sum_{u \in \states^{\tau^{\max}-1}_{1}}\sum_{l \in \states^{-1}_{\tau^L}} (u-l) \cdot  \frac{\Bar{N}_i(u,l)}{T} + \sum_{l \in \states^{-1}_{\tau^L}} \sum_{u \geq \tau^{\max}} (\tau^{\max}-l) \cdot  \frac{\Bar{N}_i(u,l)}{T}\\
        &\leq \sum_{u \in \states^{\tau^{\max}-1}_{1}}\sum_{l \in \states^{-1}_{\tau^L}} (u-l) \cdot  \frac{\Bar{N}_i(u,l)}{T} + \sum_{l \in \states^{-1}_{\tau^L}} \sum_{u \geq \tau^{\max}} (u-l) \cdot  \frac{\Bar{N}_i(u,l)}{T}\\
        &=\frac{1}{T}\sum_{u > 0}\sum_{l \in \states_{\tau^L}^{-1}} (u-l) \cdot \Bar{N}_i(u,l) \leq 1,
    \end{align*}
    which implies the feasibility of $\Bar{\x}$ w.r.t. constraints \eqref{lp:arm}. 
    
    Finally, since the aggregated payoff of the recurrent interval $I(u,l)$ for any arm $i$ is given by $q_i(u,l)$, the total collected payoff from all arms is 
    $$\overline{\opt}(T) = \sum_{i \in \A} \sum_{u>0}\sum_{l \in \states_{\tau^L}^{-1}} q_i(u,l) \Bar{N}_i(u,l).$$
    In order to conclude the proof, we notice that the objective value of $\Bar{x}$ in \eqref{lp:LP} becomes
    \begin{align*}
        \sum_{i \in \A}\sum_{u \in \states^{\tau^{\max}}_{1}}\sum_{l \in \states^{-1}_{\tau^L}} q_i(u,l) \cdot \Bar{x}_{i,u,l} 
        &= \sum_{i \in \A}\sum_{u \in \states^{\tau^{\max}-1}_{1}}\sum_{l \in \states^{-1}_{\tau^L}} q_i(u,l)  \frac{\Bar{N}_i(u,l)}{T} + \sum_{i \in \A}\sum_{l \in \states^{-1}_{\tau^L}} \sum_{u \geq \tau^{\max}} q_i(\tau^{\max},l)  \frac{\Bar{N}_i(u,l)}{T}\\
        &= \sum_{i \in \A}\sum_{u \in \states^{\tau^{\max}-1}_{1}}\sum_{l \in \states^{-1}_{\tau^L}} q_i(u,l) \frac{\Bar{N}_i(u,l)}{T} + \sum_{i \in \A}\sum_{l \in \states^{-1}_{\tau^L}} \sum_{u \geq \tau^{\max}} q_i(u,l) \frac{\Bar{N}_i(u,l)}{T}\\
        &=\frac{1}{T}\sum_{i \in \A} \sum_{u > 0}\sum_{l \in \states_{\tau^L}^{-1}} q_i(u,l) \Bar{N}_i(u,l)\\
        &= \frac{\overline{\opt}(T)}{T},
    \end{align*}
where the second equality above follows by the finite recovery assumption. 
\end{proof}

\section{Analysis of the Approximation Guarantee: Omitted Proofs} \label{apx:approx}

\restatecomparestate*

\begin{proof}
Let us fix any arm $i \in \A'$, a sampled recurrent interval $I_i = I(u_i,l_i)$, and a sampled offset $r_i$. Recall that arm $i$ is a candidate at any round $t$ (and, hence, can potentially be played) if and only if its virtual state satisfies $\beta_{I_i}(\nu_i(t)) = \bullet$. Recall, also, that the initial state of any arm (including $i$) is $1$, namely, $\tau_i(1) = 1$. We now prove by induction that for any round $t \geq \tau^{\max}$, it holds $\tau_i(t) \geq \nu_i(t)$. In particular, we prove the statement for any round $t \geq t_0$, where $t_0$ (defined below) satisfies $t_0 \leq \tau^{\max}$. 

For the base case, we first notice that if $\tau^{\max} = 1$, then by setting $t_0 = \tau^{\max} = 1$ we get that $\tau_i(t_0) = \tau_i(1) = 1 \geq \nu_i(t_0)$, by our assumption that all arms are initialized at state $1$. Let us now assume that $\tau^{\max} \geq 2$ and distinguish between two cases: (i) in the case where $\nu_i(1) \leq 1$ (i.e., the virtual state at round $t=1$ is either negative or $1$), by setting $t_0 = 1 \leq \tau^{\max}$ we immediately get that $\tau_i(t_0) = 1 > \nu_i(t_0)$, by assumption. (ii) In the case where $\nu_i(1) \geq 2$, then notice that the first time where $\beta_{I_i}(\nu_i(t')) = \bullet$ (and, hence, arm $i$ becomes a candidate) happens at $t' = u_i - \nu_i(1) + 1 \leq \tau^{\max} - 1$, by definition of a recurrent interval. Thus, by setting $t_0 = t' + 1 \leq \tau^{\max}$, we have that $\tau_i(t_0)$ is either $u_i+1$ or $-1$ (depending on whether or not it is played at round $t'$); in both cases, it holds $\tau_i(t_0) \geq -1 = \nu_i(t_0)$. Thus, we have established that $\tau_i(t_0) \geq \nu_i(t_0)$ for some round $t_0 \leq \tau^{\max}$.

Let us now assume that $\tau_i(t) \geq \nu_i(t)$ for some round $t \geq \tau^{\max}$ and prove the inductive step for round $t+1$. Clearly, in the case where $-l_i < \nu_i(t) < 0$ (i.e., the negative state corresponds to a consecutive play), then $\nu_i(t+1) = \nu_i(t) - 1$ and the actual state either ends up at $\tau_i(t+1) = \tau_i(t) - 1 \geq \nu_i(t) - 1 = \nu_i(t + 1)$ if $\tau_i(t) < 0$, or at $\tau_i(t+1) = - 1$ if $\tau_i(t) > 0$. In the case where $\nu_i(t) = -l_i$ (and, hence, the arm is not played at $t$), then it holds that $\nu_i(t+1) = 1$ and the actual state either ends up at $\tau_i(t+1) = 1 = \nu_i(t + 1)$ if $\tau_i(t) < 0$, or at $\tau_i(t+1) = \tau_i(t) + 1 \geq 1 = \nu_i(t + 1)$ if $\tau_i(t) > 0$. In the case where $0 < \nu_i(t) < u_i$ then, clearly, $\nu_i(t+1) = \nu_i(t) + 1 \leq \tau_i(t) + 1 = \tau_i(t+1)$, since the arm does not become a candidate at round $t$. Finally, for the case where $\nu_i(t) = u_i$, then it must be that either $\nu_i(t+1) = \tau_i(t+1) = -1$ if arm is played at round $t$, or $\tau_i(t+1) > \tau_i(t) \geq u_i > -1 = \nu_i(t+1)$, otherwise.
\end{proof}

\restatedistribution*
\begin{proof}
We first notice that, since the sampling of the recurrent intervals and offsets is performed independently for each arm, the random sets $\mathcal{T}^1_t, \ldots, \mathcal{T}^n_t$ must be independent. Further, by definition of our algorithm, for every $(i, I(u,l), \nu) \in U^i$, we have that
\begin{align*}
\Prob{}{\mathcal{T}^i_t = \{(i, I(u,l), \nu)\}} = \Prob{}{\nu_i(t)=\nu \;|\; I_i=I(u,l)} \cdot \Prob{}{I_i=I(u,l)} = \frac{1}{u-l} \cdot (u-l) x^*_{i, u, l} = x^*_{i, u, l},
\end{align*}
and also
\begin{align*}
\Prob{}{\mathcal{T}^i_t = \emptyset} = 1-\sum_{(i,I(u,l), \tau) \in U^i} x^*_{i, u, l}.
\end{align*}
Note that for each arm $i \in [n]$, the distribution of the set $\mathcal{T}^i_t$ is identical for every $t$, and let us denote it by $\mathcal{C}^i(\x^*)$. Given the above, the distribution of $\mathcal{T}_t$ denoted by $\mathcal{C}(\x^*)$ is independent of the time step $t$ and is equivalent to sampling $n$ sets from the independent distributions $\mathcal{C}^1(\x^*), \ldots, \mathcal{C}^n(\x^*)$ and taking their union; this satisfies the definition of a block-mutually-exclusive distribution.
\end{proof}

\restatesubmodgap*
\begin{proof}
Let $V^1, \ldots, V^n$ be the partition of $[m]$ defining the blocks of $\mathcal{C}(\y)$. Let $\mathcal{C}^1(\y), \ldots, \mathcal{C}^n(\y)$ be the block-marginals of $\mathcal{C}(\y)$ where for every $i \in [n]$, $\mathcal{C}^i(\y)$ is defined over the singletons of $2^{V_i}$ and the empty set. Recall that sampling $S \sim \mathcal{C}(\y)$ is equivalent to sampling independently $n$ sets $S^1, \ldots, S^n$ from the distributions $\mathcal{C}^1(\y), \ldots, \mathcal{C}^n(\y)$, respectively, and taking their union. Similarly, sampling $S \sim \mathcal{I}(\y)$ is equivalent to sampling independently $n$ sets $S^1, \ldots, S^n$ from the distributions $\mathcal{I}^1(\y), \ldots, \mathcal{I}^n(\y)$, respectively, and taking their union, where for every $i\in [n]$, $\mathcal{I}^i(\y)$ is the element-wise independent distribution over $V^i$ with marginals $\Prob{S \sim \mathcal{I}^i(\y)}{j \in S} = y_{j}$ for every $j \in V^i$. 

Consider now $n$ sets $S^1, \dots, S^n$ randomly sampled from $n$ independent distributions $\mathcal{D}^1, \dots, \mathcal{D}^n$ such that for each $i \in [n]$, the distribution $\mathcal{D}^i$ is either the mutually exclusive distribution $\mathcal{C}^i(\y)$ or the independent one $\mathcal{I}^i(\y)$. Consider $i \in [n]$ such that $\mathcal{D}^i \sim \mathcal{C}^i(\y)$ (if one exists), we show that switching $\mathcal{D}^i$ to $\mathcal{I}^i(\y)$ decreases the expected value of $f(\cup_{j=1}^n S^j)$. This implies the lemma by iteratively switching the mutually exclusive distributions to independent ones starting from $\mathcal{D}^1 \sim \mathcal{C}^1(\y), \dots, \mathcal{D}^n \sim \mathcal{C}^n(\y)$. Note that it is sufficient to show that for every fixed values of the rest of the sets $S^j$ for $j \neq i$, it holds that 
$$\Ex{S \sim \mathcal{C}^i(\y)}{f(\cup_{j=1}^n S^j)}
    \geq
    \Ex{S \sim \mathcal{I}^i(\y)}{f(\cup_{j=1}^n S^j)}.$$

Let us begin by proving the following result on subadditive functions\footnote{A set function $g : 2^V \rightarrow \mathbb{R}$ is subadditive if and only if for every $S, T \subseteq V$, we have $g(S \cup T) \leq g(S) + g(T)$.}. 

\begin{lemma}\label{lem:subadditive}
Let $g : 2^{V} \rightarrow \mathbb{R}$ be a subadditive function over a ground set $V$ such that $g(\emptyset) = 0$. Consider a distribution $\mathcal{D}$ over the singletons of $2^V$ and the empty set, and let $\mathcal{I}$ be the element-wise independent distribution over $V$ with same marginals, i.e., such that $\Prob{S \sim \mathcal{I}}{e \in S} = \Prob{\mathcal{D}}{\{e\}}$ for every $e \in V$. Then,
        \begin{align*}
            \Ex{S \sim \mathcal{D}}{g(S)}
            \geq
            \Ex{S \sim \mathcal{I}}{g(S)}.
        \end{align*}
\end{lemma}
\begin{proof}
Notice that since $g(\emptyset) = 0$, then for any set $S \subseteq V$, the subadditivity of $g$ implies that

$$
g(S) \leq \sum_{e \in S} g(\{e\}) = \sum_{e \in V} g(\{e\}) \cdot \event{e \in S},
$$
where $\event{E}$ is the indicator function such that $\event{E} = 1$, if $E$ holds true, and $\event{E} = 0$, otherwise. By taking expectation over $\mathcal{I}$ on the above expression and using the fact that $\mathcal{I}$ and $\mathcal{D}$ have the same marginals, we can conclude that
$$
\Ex{S \sim \mathcal{I}}{g(S)} \leq \sum_{e \in V} g(\{e\}) \cdot \Prob{S \sim \mathcal{I}}{e \in S} = \Ex{S \sim \mathcal{D}}{g(S)}.
$$
\end{proof}
Now for every $A \subseteq U$, let us denote $f(\cdot|A): 2^U \rightarrow R$ such that
$$
\quad f(S \mid A) = f(S \cup A) - f(A) \qquad \forall S \subseteq U, 
$$ denote the marginal function of $f$ and note that and $f(\emptyset|A) = 0$. It is known that since $f$ is submodular then $f(\cdot|A)$ has to be subadditive. Hence, for every fixed values of $S^j$ for $j \neq i$, by \Cref{lem:subadditive}, it holds that 
$$
    \Ex{S \sim \mathcal{C}^i(\y)}{f(S^i | \cup_{j \neq i} S^j)}
    \geq
    \Ex{S \sim \mathcal{I}^i(\y)}{f(S^i | \cup_{j \neq i} S^j)}.
    $$ which implies that $$\Ex{S \sim \mathcal{C}^i(\y)}{f(\cup_{j=1}^n S^j)}
    \geq
    \Ex{S \sim \mathcal{I}^i(\y)}{f(\cup_{j=1}^n S^j)},$$ which concludes the proof.
\end{proof}

\restatesuptolp*
\begin{proof}
We prove the lemma by explicitly constructing a distribution $\mathcal{L}(\x^*)$ under which the expectation of $g(C)$ is equal to $\lp^*$. In particular, consider the vector 
$$\mathbf{v} := \left(-lx^*_{i,u, l}\right)_{i \in \mathcal{A} \;,\; u \in \states^{ \tau^{\max}}_1 \;,\; l \in \states^{-1}_{\tau^L}}.$$ 
Constraints \eqref{lp:total} of \eqref{lp:LP} imply that the sum of the coordinates of $\mathbf{v}$ is at most $k$, while constraints \eqref{lp:arm} suggest that $\mathbf{v}$ lies in the unit hypercube. Therefore, vector $\mathbf{v}$ may be written as a convex combination of $m$ boolean vectors $\mathbf{v}^1, \mathbf{v}^2, \ldots, \mathbf{v}^m \in \{0,1\}^{n \times \tau^{max} \times (-\tau^{L})}$ such that for every $j \in [m]$ the vector $\mathbf{v}^j$ has at most $k$ ones. Let $\lambda_1, \ldots, \lambda_m \geq 0$ such that $\sum_{j=1}^m \lambda_j = 1$ and $\mathbf{v} = \sum_{j=1}^m \lambda_j \mathbf{v}^j$. We construct the following distribution $\mathcal{L}(\x^*)$ over $2^U$: to sample a set $S \sim \mathcal{L}(\x^*)$, we first sample a vector $\mathbf{w} \in \{\mathbf{v}^1, \ldots, \mathbf{v}^m\}$ where for every $j \in [m]$, the probability of sampling $\mathbf{v}^j$ is $\lambda_j$. Next, for each positive coordinate $(i,u, l)$ of $\mathbf{w}$ (i.e., such that $w_{i,u, l} = 1$), we sample a triple $T_{i,u,l}$ from the set $\{(i, I(u, l), \tau) \;|\; \beta_{I(u,l)} = \bullet\}$ uniformly at random and add $T_{i,u,l}$ to $S$. Notice that the set $S$ is by construction a subset of $U$ (the set of all possible triples). Note, further, that because $\mathbf{w}$ has at most $k$ ones, $S$ contains at most $k$ triples. 

We now claim that the distribution $\mathcal{L}(\x^*)$ constructed above has marginals $(x^*_{i,u,l})_{i,u,l}$. Indeed, fix a triple $(i, I(u,l), \nu)$. The probability that $(i, I(u,l), \nu)$ belongs to $S$ is given by
\begin{align*}
    \Prob{S \sim \mathcal{L}(\x^*)}{(i, I(u,l), \nu) \in S} 
    &= 
    \sum_{j=1}^m \Prob{}{\left.(i, I(u,l), \nu) \in S \right| \mathbf{w} = \mathbf{v}^j} \Prob{}{\mathbf{w}=\mathbf{v}^j}\\
    &= 
    \sum_{j \in [m]:v^j_{i, u, l} = 1} \Prob{}{\left.(i, I(u,l), \nu) \in S \right| \mathbf{w} = \mathbf{v}^j} \Prob{}{\mathbf{w}=\mathbf{v}^j}\\
    &= 
    \sum_{j \in [m]:v^j_{i, u, l} = 1} \Prob{}{T_{i,u,l} = (i, I(u,l), \nu) | \mathbf{w} = \mathbf{v}^j} \cdot \lambda_j\\
    &= 
    \sum_{j \in [m]:v^j_{i, u, l} = 1} \frac{\lambda_j}{-l} \\
    &= -
    \frac{1}{l}\sum_{j=1}^m \lambda_j v^j_{i, u, l} \\
    &= x^*_{i, u, l}.
\end{align*}
In order to see why the second equality above holds, note that $S$ only contains triples $(i',I(u',l'),\tau')$ such that $w_{i',u',l'} = 1$. Hence, for every $j \in [m]$ such that $v^j_{i,u,l} = 0$ and conditioned on $\mathbf{w} = \mathbf{v}^j$, the triple $(i, I(u,l), \nu)$ does not belong to $S$ and, thus, $\Prob{}{\left(i, I(u,l), \nu) \in S \right) \mid \mathbf{w} = \mathbf{v}^j} = 0$. For the third equality, note that conditioned on $\mathbf{w} = \mathbf{v}^j$, the only way that $(i, I(u,l), \nu)$ can belong to $S$ is if the triple $T_{i,u,l}$ sampled for the coordinate $i,u,l$ of $\mathbf{w}$ is $(i, I(u,l), \nu)$. Finally, the fourth equality holds because $|\{(i, I(u, l), \tau) \;|\; \beta_{I(u,l)} = \bullet\}| = -l$ and $T_{i,u,l}$ is sampled uniformly at random, while the last follows by definition of $\mathbf{v}^1, \dots, \mathbf{v}^m$.

It remains to show that the expected value of $g(C)$ when $C \sim \mathcal{L}(\x^*)$ is indeed $\lp^*$. In fact,
\begin{align*}
    \Ex{C \sim \mathcal{L}(\x^*)}{g(C)}
    &=
    \Ex{C \sim \mathcal{L}(\x^*)}{\max_{S \subseteq C, |S| \leq k} \sum_{(i,I(u, l), \tau) \in S} p_{i}(\tau)}\\
    &=
    \Ex{C \sim \mathcal{L}(\x^*)}{\sum_{(i,I(u, l), \tau) \in C} p_{i}(\tau)}\\
    &= \sum_{(i,I(u, l), \tau) \in U} p_{i}(\tau) \Prob{C \sim \mathcal{L}(\x^*)}{ (i,I(u, l), \tau) \in C}\\
    &= \sum_{(i,I(u, l), \tau) \in U} p_{i}(\tau) x^*_{i,u, l}\\
    &= \lp^*,
\end{align*}
where the second equality follows from the fact that any set $C \sim \mathcal{L}(x^*)$ has at most $k$ elements (triples).
\end{proof}

\restatemain*
\begin{proof}
Consider a time step $t \geq \tau^{\max}$ and let $\gamma_k = \left(1-\frac{k^k}{e^k k!}\right)$. Inequality \eqref{eq:virtualstates} and identity \eqref{eq:triples} imply that the expected payoff collected by \Cref{algo} at round $t$ is at least
$$
\Ex{C \sim \mathcal{C}(\x^*)}{\max_{S \subseteq C, |S| \leq k} \sum_{(i, I(u,l), \nu) \in S} p_i(\nu)} = \Ex{C \sim \mathcal{C}(\x^*)}{g_k(C)}.
$$
By combining the correlation gap lemma (see \Cref{lem:correlationgap}) with \Cref{lem:submodulargap}, we can conclude that the above quantity is at least
$$
\gamma_k \cdot \sup_{\mathcal{D}(\x)} \Ex{C \sim \mathcal{D}(\x^*)}{\max_{S \subseteq C, |S| \leq k} \sum_{(i, I(u,l), \nu) \in S} p_i(\nu)},
$$
which, by using \Cref{lem:relationtolp}, can be further lower-bounded by $\gamma_k\cdot \lp^*$. Finally, by \Cref{lem:relaxation} and noting that $-\tau^L \geq \epsilon$, the algorithm collects an average payoff at time $t$ of at least
\begin{align*}
    \gamma_k \cdot \lp^* 
    &\geq 
    \left(1-\frac{1}{1-\tau^L}\right) \cdot \gamma_k \cdot \frac{\opt(T)}{T} - \gamma_k \cdot \frac{n}{T}
    \\ 
    &\geq \left(1-\epsilon\right) \cdot \gamma_k \cdot \frac{\opt(T)}{T} - \mathcal{O}\left(\frac{n}{T}\right).
\end{align*}

Therefore, the payoff collected by the algorithm during time steps $t \geq \tau^{\max}$ (where the above inequality holds) can be lower-bounded as
\begin{align*}
    (1-\epsilon)\cdot \gamma_k \cdot \opt(T)\left(1 - \frac{\tau^{\max}}{T}\right) - \mathcal{O}\left( n \cdot \left(1 - \frac{\tau^{\max}}{T}\right)\right),
\end{align*}
which can be further bounded as
\begin{align*}
    (1-\epsilon)\cdot \gamma_k \cdot \opt(T) - \mathcal{O}\left( n + \tau^{\max} \cdot k\right),
\end{align*}
using the fact that $\opt(T) \leq k T$. This completes the proof. 
\end{proof}

\section{Online Learning Adaptation: Omitted Proofs}

\label{apx:learning}

Our online learning adaptation is based on an Explore-Then-Commit scheme, where we first learn the unknown mean payoffs up to a certain precision $\eta$ and then run \Cref{algo} using the estimated payoffs. 

\paragraph{Robustness of \Cref{algo} to perturbations in the mean payoffs.} Recall that \Cref{algo} first computes an optimal solution $\x^*$ to \eqref{lp:LP} and then uses $\x^*$ to compute a feasible arm-playing schedule. It is not hard to verify that the algorithm does not require the monotonicity of the payoff functions in order to produce a feasible solution, since this assumption is only required for proving its approximation guarantee. 

The following lemma bounds the regret of \Cref{algo} when it runs on approximate (not necessarily monotone) values of the mean payoff functions:

\begin{restatable}{lemma}{restaterobustness} \label{lem:robustness}
    Let $(\hat{p}_i(\tau))_{i, \tau}$ be $\eta$-estimates of the mean payoffs  $(p_i(\tau))_{i,\tau}$ such that $|\hat{p}_i(\tau) - p_i(\tau)| \leq \eta$ for every $i \in \A$ and $\tau \in \states_{1}^{\tau^{\max}} \cup \states_{\tau^{L}}^{-1}$ for some $\eta \in (0,1)$. Let $\hat{R}(T)$ denote the total payoff collected by the algorithm when it runs with the estimates $(\hat{p}_i(\tau))_{i, \tau}$.
    Then the regret suffered by \Cref{algo} over $T$ rounds, given by
    $
    \left(1-\epsilon\right)\gamma_k \opt(T) - \hat{R}(T)
    $, is at most $\mathcal{O}\left(\eta k T + k \tau^{\max} + n\right)$.
\end{restatable}
\begin{proof}
Assuming that \Cref{algo} has access to $\eta$-estimates $(\hat{p}_i(\tau))_i$ of the payoff function of the arms, the algorithm computes an optimal solution $\hat{\x}^*$ to \eqref{lp:LP} with expected aggregated payoffs
$\hat{q}_i(u,l) = \hat{p}_i(u) + \sum_{\tau=l+1}^{-1}\hat{p}_i(\tau)$. Let $\x^*$ be the optimal solution of \eqref{lp:LP} with respect to the true expected aggregated payoffs $q_i(u,l)$. Note that for every $i \in \A, u \in \states^{\tau^{\max}}_1$ and $l \in \states_{\tau^{L}}^{-1}$, it holds that $\hat{q}_i(u,l) \geq q_i(u,l) +l\eta$. The expected gain of the algorithm at any time $t \geq \tau^{\max}$ running with the $\eta$-estimates, denoted by $\widehat{R}(t)$, is such that
\begin{align*}
    \widehat{R}(t) 
    &= \Ex{}{\max_{S \subseteq C_t, |S| \leq k} \sum_{(i,I(u, l), \tau) \in S} \hat{p}_{i}(\tau_{i}(t))}\\
    &\geq \Ex{}{\max_{S \subseteq C_t, |S| \leq k} \sum_{(i,I(u, l), \tau) \in S} p_{i}(\tau_{i}(t)) } - \eta k\\
    &\geq \Ex{}{\max_{S \subseteq C_t, |S| \leq k} \sum_{(i,I(u, l), \tau) \in S} p_{i}(\nu_{i}(t))} - \eta k\\
    &= \Ex{C \sim \mathcal{C}(\hat{\x}^*)}{\max_{S \subseteq C, |S| \leq k} \sum_{(i,I(u, l), \tau) \in S} p_{i}(\tau)} - \eta k\\
    &\geq \Ex{C \sim \mathcal{C}(\hat{\x}^*)}{\max_{S \subseteq C, |S| \leq k} \sum_{(i,I(u, l), \tau) \in S} \hat{p}_{i}(\tau) } - 2\eta k,
\end{align*}
where the first inequality holds because  $\hat{p}_{i}(\tau) \geq p_{i}(\tau) - \eta$ for all $i \in \A$ and $\tau \in \states^{\tau^{\max}}_1 \cup \states_{\tau^{L}}^{-1}$, the second inequality holds by \Cref{lem:comparestate}, and the last inequality follows from $p_{i}(\tau) \geq \hat{p}_{i}(\tau) - \eta$ for all $i \in \A$ and $\tau \in \states^{\tau^{\max}}_1 \cup \states_{\tau^{L}}^{-1}$. At this point, we can simply follow the lines of the proof of \Cref{thm:main}, given that the rest of the proof does not rely on the monotonicity of the payoff functions. Thus, we get
\begin{align*}
    \widehat{R}(t) 
    &\geq \Ex{C \sim \mathcal{C}(\hat{\x}^*)}{\max_{S \subseteq C, |S| \leq k} \sum_{(i,I(u, l), \tau) \in S} \hat{p}_{i}(\tau) } - 2\eta k\\
    &\geq \gamma_k \sum_{i \in \A}\sum_{u \in \states^{\tau^{\max}}_{1}}\sum_{l \in \states^{-1}_{\tau^L}} \hat{q}_i(u,l) \cdot \hat{x}_{i,u,l} - 2\eta k\\
    &\geq \gamma_k\sum_{i \in \A}\sum_{u \in \states^{\tau^{\max}}_{1}}\sum_{l \in \states^{-1}_{\tau^L}} \hat{q}_i(u,l) \cdot x^*_{i,u,l} - 2\eta k\\
    &\geq \gamma_k\sum_{i \in \A}\sum_{u \in \states^{\tau^{\max}}_{1}}\sum_{l \in \states^{-1}_{\tau^L}} q_i(u,l) \cdot x^*_{i,u,l} - \gamma_k\eta \sum_{i \in \A}\sum_{u \in \states^{\tau^{\max}}_{1}}\sum_{l \in \states^{-1}_{\tau^L}} -l x^*_{i,u,l} - 2\eta k\\
    &= \gamma_k\lp^* - (2+\gamma_k)\eta k\\
    &\geq \left(1-\epsilon\right)\gamma_k \frac{\opt(T)}{T} - \gamma_k \frac{n}{T} - \mathcal{O}\left(\eta k\right),
\end{align*}
where the second inequality holds by following the lines of the proof for \Cref{thm:main} and the third by optimality of $\hat{\x}^*$ for \eqref{lp:LP} with the aggregated payoffs $\hat{q}_i(u,l)$. The penultimate inequality holds because $\hat{q}_i(u,l) \geq q_i(u,l) + l \eta$ for all $i \in \A, u \in \states^{\tau^{\max}}_1$ and $l \in \states_{\tau^{L}}^{-1}$ and the last using \Cref{lem:relaxation}. The proof of the lemma follows from summing the above inequality over all $t \geq \tau^{\max}$. \end{proof}

\paragraph{Sample complexity of the collecting $\eta$-estimates.} We now upper-bound the number of time steps required to get $\eta$-estimates of the mean payoffs $p_i(\tau)$ for every arm $i$ and state $\tau$, with high probability. In particular, we prove the following lemma:
\begin{restatable}{lemma}{restatechernoff} \label{lem:chernoff}
    For any $\eta, \delta \in (0,1)$, let $\hat{p}_i(\tau)$ be the empirical mean of $m$ samples drawn from $p_i(\tau)$ for every arm $i \in \A$ and $\tau \in \states_{1}^{\tau^{\max}} \cup \states_{\tau^{L}}^{-1}$. By setting $m = \frac{1}{2\eta^2} \ln\left(\frac{2n(\tau^{\max}-\tau^{L})}{\delta}\right)$ then, with probability at least $1-\delta$, it holds that $|\hat{p}_i(\tau) - p_i(\tau)| \leq \eta$ for all $i$ and $\tau$. Moreover, it is possible to collect $m$ samples from each pair of arm and state within $\frac{nm((\tau^{\max})^2 - \tau^{L}+2)}{k}$ consecutive time steps.
\end{restatable}

\begin{proof}
The first part of the lemma is a direct consequence of Hoeffding inequality applied to the bounded i.i.d. samples of $p_i(\tau)$. In particular, for every $i$ and $\tau$ and after $m$ samples, it holds that 
$$
\Prob{}{|\hat{p}_i(\tau)-p_i(\tau)| > \eta} \leq 2e^{-2m\eta^2}=\frac{\delta}{n(\tau^{\max} - \tau^L)}.
$$
The result follows by a union bound over all arms $i \in \A$ and states $\tau \in \states^{\tau^{\max}}_1 \cup \states_{\tau^{L}}^{-1}$.

In order to collect $m$ samples out of every $p_i(\tau)$ for every $i \in \A$ and $\tau \in \states^{\tau^{\max}}_1 \cup \states_{\tau^{L}}^{-1}$, we do the following: consider the arms $\{1, \ldots, k\}$, for each arm $i \in \{1, \ldots, k\}$ and starting from the beginning of the time horizon, sample arm $i$ repeatedly $m$ times at state $\tau=\tau^{\max}$, then $m$ times at state $\tau=\tau^{\max}-1$ and so on. After the last sample at state $\tau=1$, repeat for $m$ times a sequence of $-\tau^L+1$ plays followed by a non-play. Note that this allows to collect $m$ samples of arm $i$ at every state $\tau \in \states^{\tau^{\max}}_1 \cup \states_{\tau^{L}}^{-1}$. The total number of steps needed to collect all of these samples is 
$$
\Delta = m \sum_{j=1}^{\tau^{\max}} j + m (-\tau^L+2) = \frac{m\tau^{\max}(\tau^{\max}+1)}{2}+m(-\tau^L+2).
$$
Similarly, use the next $\Delta$ time steps to collect $m$ samples out of every $p_i(\tau)$ for the arms $i \in \{k+1, \ldots, 2k\}$ and so on. This above process yields a feasible schedule (as no more than $k$ arms are sampled at every time step $t$) which collects (at least) $m$ samples for each pair of arm and state in at most $\frac{n\Delta}{k} \leq \frac{nm((\tau^{\max})^2-\tau^L+2)}{k}$ time steps.
\end{proof}

\paragraph{Online algorithm for the bandit problem.} By combining the above results, we are now ready to state our algorithm for the bandit setting of $k$-$\mlsd$. We remark that, by a standard application of the doubling trick (which comes at the small cost of a polylogarithmic factor in the regret), we can assume w.l.o.g. that the time horizon is known to the player a priori. 

Let $\eta, \delta \in (0,1)$ and let $m = \frac{1}{2\eta^2} \ln\left(\frac{2n(\tau^{\max}-\tau^{L})}{\delta}\right)$. Our algorithm is an Explore-Then-Commit variant of \Cref{algo}. The first $\frac{nm((\tau^{\max})^2 - \tau^{L}+2)}{k}$ rounds of the algorithm are used to collect $m$ i.i.d. samples from $p_i(\tau)$ for every arm $i$ and state $\tau$, as described in \Cref{lem:chernoff}. Then, again by \Cref{lem:chernoff}, the empirical means of these samples $(\hat{p}_i(\tau))_{i,\tau}$ are $\eta$-estimates of the mean payoffs $(p_i(\tau))_{i, \tau}$ with probability at least $1-\delta$. In the remainder of the rounds, the online algorithm runs \Cref{algo} with the estimates $(\hat{p}_i(\tau))_{i,\tau}$ as input. By a proper tuning of the parameters $\eta$ and $\delta$, we recover the regret bounds of \Cref{thm:regret}.

\begin{proof}[Proof of \Cref{thm:regret}]
    Let $\eta, \delta \in (0,1)$ to be specified later, and let $m = \frac{1}{2\eta^2}\ln(\frac{2n(\tau^{\max}-\tau^L)}{\delta})$. 
    The exploration phase consists of collecting $m$ samples from each arm $i$ and state $\tau$, which, by \Cref{lem:chernoff} can be done in the first $\frac{nm((\tau^{\max})^2-\tau^L+2)}{k}$ time steps. At each time step, a payoff of at most $1$ is collected by the optimal solution from each arm, which implies a total accumulated regret of $nm((\tau^{\max})^2-\tau^L+2)$ in the exploration phase. At the end of the exploration phase, the empirical means of the collected samples $\hat{p}_i(\tau)$ give $\eta$-estimates of the true mean payoffs $p_i(\tau)$ with probability at least $1- {\delta}$. The exploitation phase consists of running \Cref{algo} using the $\eta$-estimates $\hat{p}_i(\tau)$. This implies, by \Cref{lem:robustness}, that with probability at least $1-\delta$, the algorithm suffers a regret of at most $\mathcal{O}(kT\eta + k\tau^{\max} + n)$ in the exploitation phase. Finally, the regret accumulated from the sample paths of the algorithm where the concentration bounds do not hold is at most $\mathcal{O}(kT\delta)$.
    Therefore, the total regret accumulated by the algorithm over the whole time horizon is such that
    $$
    \text{Reg}(T) = \mathcal{O}\left(nm((\tau^{\max})^2-\tau^L+2) + (kT\eta + k\tau^{\max} + n) + kT\delta\right).
    $$
    By setting $\eta = \sqrt[3]{\frac{n((\tau^{\max})^2-\tau^L+2)\ln (2n(\tau^{\max}-\tau^L)T)}{2kT}}$ and $\delta=\frac{1}{T}$, we can upper-bound the above regret as 
    \begin{align*}
        \text{Reg}(T) 
        &= 
        \mathcal{O}\left(n^{\frac{1}{3}}k^{\frac{2}{3}}((\tau^{\max})^2-\tau^L)^{\frac{1}{3}}\ln^{\frac{1}{3}}((\tau^{\max}-\tau^L)T) \cdot T^{\frac{2}{3}} + k \tau^{\max} + n\right)\\
        &= \mathcal{O}\left(n^{\frac{1}{3}}k^{\frac{2}{3}}\left((\tau^{\max})^2+\frac{1}{\epsilon}\right)^{\frac{1}{3}}\ln^{\frac{1}{3}}\left(\left(\tau^{\max}+\frac{1}{\epsilon}\right)T\right) \cdot T^{\frac{2}{3}} + k \tau^{\max} + n\right).
    \end{align*}
\end{proof}

\section{Additional Remarks}

\subsection{Tightness of the Approximation Analysis} \label{sec:tightexample}
We now provide an example to show that the long-run approximation guarantee of \Cref{algo}, provided in \Cref{thm:main}, is tight (up to the $(1-\epsilon)$-factor). 

Consider an instance of $k$-$\mlsd$ with $n = m \cdot k$ arms for some large integer $m$. We assume that all arms have the same payoff function, defined as follows: the payoff is $0$ for every state $\tau \leq m - 1$ and becomes $1$ for every state $\tau \geq m$. It can be easily verified that the asymptotically optimal solution to the above instance it to partition the arms into $m$ ``batches'', each containing $k$ arms, and then play a different batch at each round in a round-robin manner. Notice that this leads to a periodic arm-pulling schedule with period $m$ (the number of batches). Independently of the initial state of each arm and after at least $m$ time steps, there always exist exactly $k$ arms at each round (contained in a single batch) which are at state $\tau = m$ and thus have payoff $1$. Playing these arms at each round gives a long-run optimal average payoff of $k$. 

Let us now focus on the behavior of \Cref{algo} on the above instance. By analyzing the optimal solution to \eqref{lp:LP} in that case and by construction our sampling procedure for recurrent intervals and offsets, it can be verified that, at any round, an arm is a candidate with probability equal to $\frac{1}{m}$. Given that all arms are identical and assuming w.l.o.g. that all arms are initialized at state $\tau = m$, at each round $t$ our algorithm collects the minimum between the number of candidate arms and $k$ (breaking ties arbitrarily in the case where $|C_t|>k$). Hence, the associated payoff at each round is given by $\mathbb{E}[\min\{X,k\}]$, where $X$ is a binomial random variable with parameters $n$ (number of arms) and $\frac{1}{m} = \frac{k}{n}$ (the probability an arm is a candidate). By taking the limit $n \to \infty$, it can be proved (see, e.g., Lemma 4.2. in \cite{yan11}) that the average payoff collected by our algorithm over the optimal one becomes
\begin{align*}
    \lim_{n \to \infty} \frac{\mathbb{E}[\min\{X,k\}]}{k} = \frac{k - \frac{k^{k+1}}{e^k\cdot k!}}{k} = 1 - \frac{k^k}{e^k \cdot k!}, 
\end{align*}
which matches exactly the guarantee of our algorithm (modulo the $1-\epsilon$ factor).

\subsection{Continuous Relaxations Based on States} 
\label{apx:badrelaxations}

The existing algorithms from the recharging bandits literature construct relaxations based on the fraction of time an arm is played under a specific state (or ``delay'', using the terminology of these works). In \cite{SLZZ21} Simchi-Levi et al. use such a relaxation to construct purely periodic policies, namely, policies where each arm is repeatedly played only under a specific delay. Similarly, the randomized algorithm of Papadigenopoulos et al. \cite{papadigenopoulos2022nonstationary} allows each arm to be played only in rounds that are integer multiples of a unique arm-specific delay (which they call ``critical''). In the $k$-$\mlsd$ setting, however, it is impossible to repeatedly play an arm under (and only under) a state $\tau$ for $\tau<0$. Indeed, between two plays of under a state $\tau < 0$, any algorithm must necessarily play the arm under at least one positive state $\tau' > 0$ and all the negative states in $\{\tau+1, \ldots, -1\}$, assuming that $\tau < -1$. This begs the question of whether a restriction to positive states (that can be periodically played) is sufficient in the case of $k$-$\mlsd$ bandits. The following example shows that this is not the case:

\begin{example}
Consider an instance of $1$-$\mlsd$ with a single arm $i=1$ and an infinite time horizon. The payoff function of the arm is given by
\begin{align*}
    p_1(\tau) 
    = 
    \left\{
        \begin{matrix}
            1 \quad \tau \geq -1, \\
            0 \quad \tau < -1. \\
        \end{matrix}
    \right.
\end{align*}
It is not hard to verify that the {\em unique} optimal strategy in the above instance is to periodically repeat the sequence (play, play, non-play), starting from $t=1$. Notice that this strategy collects an average payoff of $\nicefrac{2}{3}$ and consists of playing the arm under both positive (+1) and negative (-1) states.
\end{example}

\end{document}